\documentclass[letterpaper, 10 pt, conference]{ieeeconf}

\IEEEoverridecommandlockouts                              

\usepackage[noadjust]{cite}
\usepackage[utf8]{inputenc}
\usepackage[english]{babel}
\usepackage[ruled,linesnumbered]{algorithm2e}
\usepackage{amsmath}
\usepackage{amssymb}
\usepackage{array}
\usepackage{booktabs}
\usepackage{caption}
\usepackage{dblfloatfix}
\usepackage{enumerate}
\usepackage{gensymb}
\usepackage{graphicx}
\usepackage{multirow}
\usepackage{indentfirst}
\usepackage{tabularx}
\usepackage{threeparttable}
\usepackage{subcaption}
\usepackage{mathtools}

\usepackage{color}
\usepackage{hyperref}
\usepackage{kotex}
\usepackage{diagbox}

\usepackage{amsthm}
\newtheorem{theorem}{Theorem}
\newtheorem{lemma}{Lemma}

\title{\LARGE \bf
Online Distributed Trajectory Planning for Quadrotor Swarm with Feasibility Guarantee using Linear Safe Corridor
}
\author{Jungwon Park, Dabin Kim, Gyeong Chan Kim, Dahyun Oh and H. Jin Kim$^{1}$
\thanks{$^{1}$The authors are with the Department of Mechanical and Aerospace Engineering, Seoul National University (SNU), and Automation and Systems Research Institute (ASRI), Seoul 08826, South Korea
        {\tt\small \{qwerty35, dabin404, skykim0609, qlass33, hjinkim\}@snu.ac.kr}}%
}

\begin{document}

\maketitle

\begin{abstract}
This paper presents a new online multi-agent trajectory planning algorithm that guarantees to generate safe, dynamically feasible trajectories in a cluttered environment.
The proposed algorithm utilizes a linear safe corridor (LSC) to formulate the distributed trajectory optimization problem with only feasible constraints, so it does not resort to slack variables or soft constraints to avoid optimization failure.
We adopt a priority-based goal planning method to prevent the deadlock without an additional procedure to decide which robot to yield.
The proposed algorithm can compute the trajectories for 60 agents on average 15.5 ms per agent with an Intel i7 laptop and shows a similar flight distance and distance compared to the baselines based on soft constraints.
We verified that the proposed method can reach the goal without deadlock in both the random forest and the indoor space, and we validated the safety and operability of the proposed algorithm through a real flight test with ten quadrotors in a maze-like environment.
\end{abstract}

\section{INTRODUCTION} 
The trajectory planning algorithm must guarantee collision avoidance and dynamical feasibility when operating a quadrotor swarm, as one accident can lead to total failure.
Also, the need to handle a large number of agents requests a scalable online distributed algorithm.
However, there are two main challenges to generate a safe trajectory online.
First, it is difficult to find feasible collision avoidance constraints if the agents are crowded in a narrow space. Many researchers adopt soft constraints to avoid an infeasible optimization problem, but many of them cannot guarantee collision avoidance \cite{luis2020online, kandhasamy2020scalable, zhou2020ego,zhou2021decentralized}.
Second, if the obstacles have a non-convex shape, it is difficult to solve deadlock even in sparse environments using distributed planning approaches.
The centralized methods like conflict-based search (CBS) \cite{sharon2015conflict} can solve this problem, but they are not appropriate for online planning due to their long computation time. 

In this letter, we present an online trajectory planning algorithm that guarantees feasibility of the optimization problem.
The key idea of the proposed method is a linear safe corridor (LSC), which allows the feasible set of collision avoidance constraints by utilizing the trajectory at the previous step.
We design the LSC to combine the advantages of the buffered Voronoi cell (BVC) \cite{zhou2017fast} and the relative safe flight corridor (RSFC) \cite{jungwon2020efficient}.
First, the LSC can prevent collision between agents without non-convex constraints or slack variables, similar to the BVC.
Second, since the LSC has a similar form to the RSFC, the LSC can reflect the previous solution to the current optimization problem. Hence it can reduce the flight time and distance to a similar level to on-demand collision avoidance \cite{luis2020online} or discretization-based method \cite{zhou2020ego} while guaranteeing the collision avoidance which was not done in the previous works. 
By adopting the LSC, the proposed algorithm always returns a safe trajectory without optimization failure.
Furthermore, we introduce priority-based goal planning to solve deadlock in maze-like environments. Despite the limited performance in very dense environments due to the nature of a distributed approach, it can find the trajectory to the goal in sparse environments using the previous trajectories to decide which robot to yield.
We compared the proposed algorithm with state-of-the-art methods, BVC \cite{zhou2017fast}, DMPC \cite{luis2020online}, EGO-Swarm \cite{zhou2020ego} in simulations, and we conducted flight tests to demonstrate that the proposed algorithm can be applied to real-world settings. (See Fig. \ref{fig: flight test}). 
We will release the source code of this work in \url{https://github.com/qwerty35/lsc_planner}.

We summarize the main contributions as follows.
\begin{itemize}
\item An online multi-agent trajectory planning algorithm that guarantees to generate safe, feasible trajectory in cluttered environments without potential optimization failure by using the linear safe corridor.
\item Linear safe corridor that fully utilizes the previous solution to ensure collision avoidance and feasibility of the constraints.
\item Decentralized priority-based goal planning that prevents the deadlock in a sparse maze-like environment without an additional procedure to decide which robot to yield.
\end{itemize}


\begin{figure}[t]
\centering
\includegraphics[width = 0.8\linewidth]{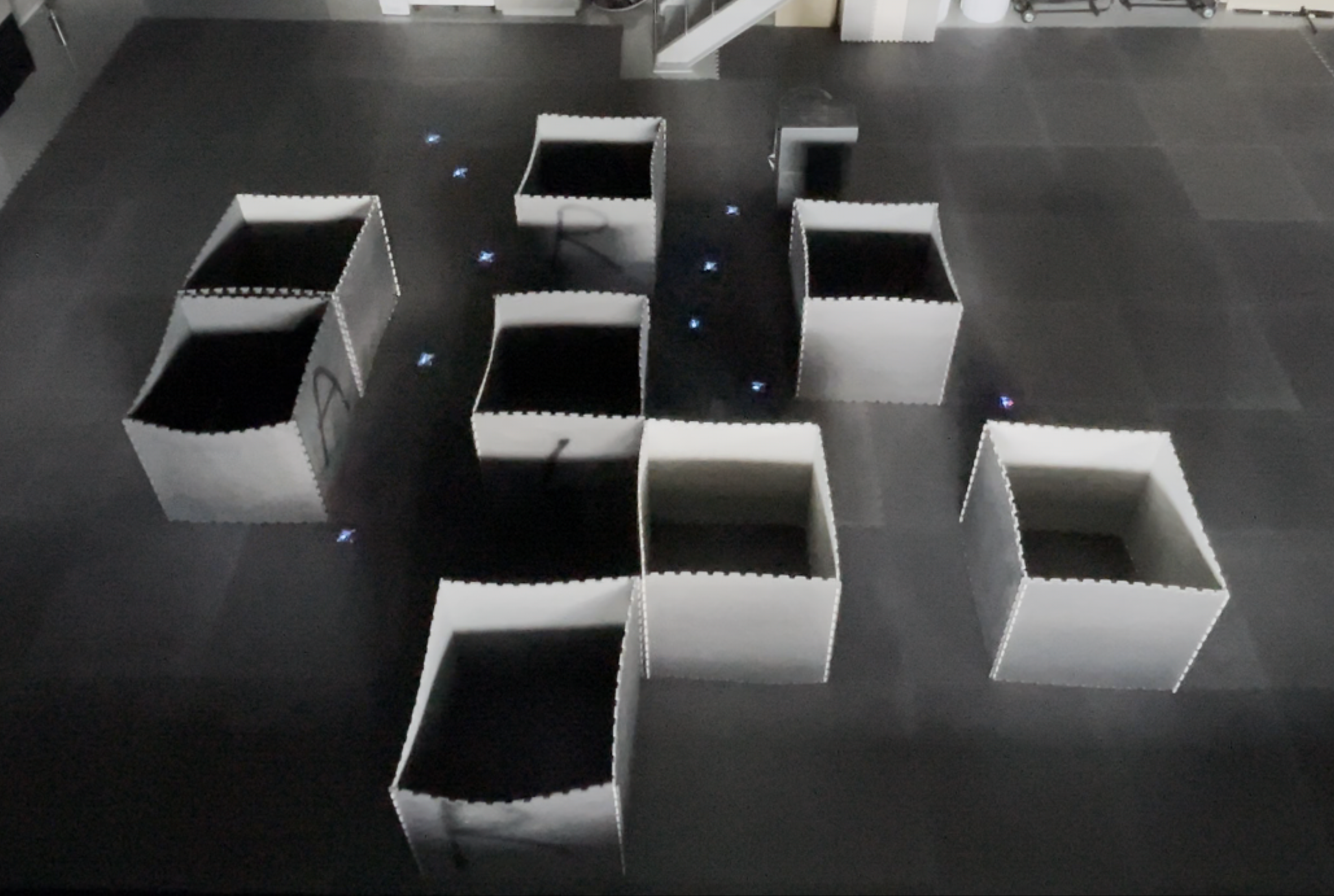}
\caption{
Experiment with 10 quadrotors in the maze-like environment.
}
\label{fig: flight test}
\vspace{-5mm}
\end{figure}

\section{RELATED WORK}
\subsection{Online Multi-Agent Collision Avoidance}

To achieve online planning, many recent works adopt discretization method, which apply the constraints only at some discrete points, not the whole trajectory \cite{luis2020online,kandhasamy2020scalable,zhou2020ego, zhou2021decentralized}.
However, this approach can occur the collision if the interval between discrete points is too long. Moreover, most of the works utilize slack variables or soft constraints to avoid the infeasible optimization problem, which increases the probability of collisions.
A velocity obstacle (VO)-based method \cite{van2011reciprocal,arul2020dcad} can be considered, but they can cause a collision either since the agents do not follow the constant velocity assumption.
On the contrary, the proposed algorithm guarantees collision avoidance for all trajectory points, and it ensures feasibility of the optimization problem without the slack variables or soft constraints.

There are several works that can guarantee both collision avoidance and feasibility.
In \cite{zhou2017fast}, buffered Voronoi cell (BVC) is presented to confine the agents to a safe region, and in \cite{wang2017safety}, a hybrid braking controller is applied for safety if the planner faces feasibility issues. However, these methods require conservative collision constraints. 
On the other hand, the proposed method employs less conservative constraints because it constructs the LSC considering the previous trajectories.
The authors of \cite{tordesillas2021mader} suggest an asynchronous planning method using a collision check-recheck scheme. This method guarantees feasibility by executing only verified trajectories through communication. However, it often takes over few seconds for replanning since it blocks the update if it receives the other agent's trajectories during the planning.
Conversely, the proposed method uses a synchronized replanning method, which takes constant time to update trajectories. 

Compared to our previous works \cite{jungwon2020efficient, park2020online}, the proposed method does not require a centralized grid-based path planner or slack variables for feasibility. Thus, the proposed algorithm guarantees inter-agent collision avoidance without soft constraints while enabling online distributed planning.

\subsection{Decentralized Deadlock Resolution}
A deadlock can occur if the agents are crowded in a narrow space without a centralized coordinator.
One approach to solve this is a right-hand rule \cite{zhu2019b}, which rotates the goal to the right when the agent detects the deadlock. However, it often causes another deadlock in cluttered environments.
In \cite{jager2001decentralized}, the agents are sequentially asked to plan an alternative trajectory until the deadlock is resolved.
However, it requires an additional procedure to decide which robot to yield. On the other hand, the proposed method does not require such a decision process because it determines the priority by the distance to the goal.
The authors of \cite{desaraju2012decentralized} present a cooperative strategy based on bidding to resolve the deadlock. However, this method takes more time for replanning as the number of agents increases because it updates only the trajectory of the winning bidder.
A similar approaches to our proposed method can be found in \cite{csenbacslar2019robust}. It utilizes the grid-based planner to avoid conflict between agents. However, it often fails to find the discrete path in a compact space because it treats all other agents as static obstacles.
To solve this, the proposed method imposes the priority on the agents so that the agents ignore the lower-priority agents while planning the discrete path.

\section{PROBLEM FORMULATION}
\label{sec: problem formulation}
In this section, we formulate a distributed trajectory optimization problem that guarantees feasibility.
We suppose that there are $N$ agents in an obstacle environment $\mathcal{W}$, and each agent $i$ has the mission to move from the start point $\textbf{s}^{i}$ to the goal point $\textbf{g}^{i}$.
Our goal is to generate a safe, continuous trajectory so that all agents can reach the goal point without collision or exceeding the agent's dynamical limits. 
Throughout this paper, the calligraphic uppercase letter means set, the bold lowercase letter means vector, and the italic lowercase letter means scalar value. The superscript of each symbol shows the index of the associated agent.

\subsection{Assumption}
We assume that all agents start their mission with the collision-free initial configuration and they can follow the trajectory accurately.
We suppose that all agents can share their current positions, goal points and the trajectories planned in the previous step without delay or loss of communication.
We also assume that the free space of the environment is given as prior knowledge.

\subsection{Trajectory Representation}
Due to the differential flatness of quadrotor dynamics \cite{mellinger2011minimum}, we represent the trajectory of agent $i$ at the replanning step $k$, $\textbf{p}^{i}_{k}(t)$, with the piecewise Bernstein polynomial:
\begin{equation}
\begin{alignedat}{2}
    \textbf{p}^{i}_{k}(t) = 
    \begin{cases} 
    \ \sum_{l=0}^{n}\textbf{c}^{i}_{k,1,l}b_{l,n}(\tau(t,k,1))   & t \in [T_{k}, T_{k+1}] \\
    \ \sum_{l=0}^{n}\textbf{c}^{i}_{k,2,l}b_{l,n}(\tau(t,k,2))   & t \in [T_{k+1}, T_{k+2}] \\
    \ \vdotswithin{=}                                                 & \vdotswithin{\in} \\    
    \ \sum_{l=0}^{n}\textbf{c}^{i}_{k,M,l}b_{l,n}(\tau(t,k,M))   & t \in [T_{k+M-1}, \\ 
    & \:\:\:\:\:\:\:\: T_{k+M}] \\
    \end{cases}
\end{alignedat}
\label{eq: trajectory representation}
\end{equation}
where $n$ is degree of polynomial, $\textbf{c}^{i}_{k,m,l} \in \mathbb{R}^3$ is the $l^{th}$ control point of the $m^{th}$ segment of the trajectory, $b_{l,n}(t)$ is Bernstein basis polynomials, $T_{k}$ is the start time of the replanning step $k$, $T_{k+m}=T_{k}+m\Delta t$, $\Delta t$ is duration of each segment, and $\tau(t,k,m)=(t-T_{k+m-1})/\Delta t$.
Thus, the control points of the piecewise polynomial, $\textbf{c}^i_{k,m,l}$ for $m=1,\cdots,M$ and $l=0,\cdots,n$, are the decision variables of our optimization problem.
Here, the Bernstein basis can be replaced with another one that satisfies the convex hull property, such as \cite{tordesillas2020minvo}.
In this paper, we denote the $m^{th}$ segment of the trajectory as $\textbf{p}^{i}_{k,m}(t)$ and the vector that contains all control points of the trajectory as $\textbf{c}^{i}_{k} \in \mathbb{R}^{3M(n+1)}$.

\subsection{Objective Function}
\subsubsection{Error to goal}
We minimize the distance between the agent and the current goal point $\textbf{g}^{i}_{curr}$ so that the agent can reach the final goal point.
\begin{equation}
    J^{i}_{err} = \sum\limits_{m=1}^{M} w_{e,m} \|\textbf{p}^{i}_{k}(T_{m})-\textbf{g}^{i}_{curr}\|^{2}
\label{eq: objective function error to goal}
\end{equation}
where $w_{e,m=1,\cdots,M} \geq 0$ is the weight coefficient and $\|\cdot\|$ is the Euclidean norm.

\subsubsection{Derivatives of trajectory}
We minimize the derivatives of trajectory to penalize aggressive maneuvers of the agent.
\begin{equation}
\begin{aligned}
    J^{i}_{der} = \sum\limits_{r=1}^{n}w_{d,r}\int_{T_0}^{T_M}\left\|\frac{d^{r}}{dt^{r}}\textbf{p}^{i}_{k}(t)\right\|^{2}dt
\label{eq: objective function derivative of trajectory}
\end{aligned}
\end{equation}
where $w_{d,r=1,\cdots,n} \geq 0$ is the weight coefficient. In this work, we set the weights to minimize the jerk.

\subsection{Convex Constraints}
The trajectory should satisfy the initial condition to match the agent's current position, velocity and acceleration.
Also, the trajectory should be continuous up to the acceleration for smooth flight.
Due to the dynamical limit of quadrotor, the agents should not exceed their maximum velocity and acceleration.
We can formulate these constraints in affine equality and inequality constraints using the convex hull property of the Bernstein polynomials \cite{zettler1998robustness}:
\begin{equation}
    A_{eq}\textbf{c}^{i}_{k} = \textbf{b}_{eq}
\label{eq: equality constraints}
\end{equation}
\begin{equation}
    A_{dyn}\textbf{c}^{i}_{k} \preceq \textbf{b}_{dyn}
\label{eq: dynamic feasible constraints}
\end{equation}

\subsection{Collision Avoidance Constraints}
\label{subsec: collision avoidance constraints}
\subsubsection{Obstacle avoidance constraint}
To avoid collision with static obstacle, we formulate the constraint as follows:
\begin{equation}
    \textbf{p}^{i}_{k}(t) \oplus \mathcal{C}^{i,o} \subset \mathcal{F}, \:\:\: \forall t \in [T_{k}, T_{k+M}]
\label{eq: obstacle collision avoidance constraint}
\end{equation}
\begin{equation}
  \mathcal{C}^{i,o}=\{\textbf{x} \in \mathbb{R}^3 \mid \|\textbf{x}\| \leq r^{i}\}
\label{eq: obstacle collision model}
\end{equation}
where $\oplus$ is the Minkowski sum, $\mathcal{F}$ is the free space of the environment, and $\mathcal{C}^{i,o}$ is an obstacle collision model that has sphere shape, and $r^i$ is the radius of agent $i$.

\subsubsection{Inter-collision avoidance constraint}
We can represent the collision avoidance constraint between the agents $i$ and $j$ as follows:
\begin{equation}
     \textbf{p}^{i}_{k}(t) - \textbf{p}^{j}_{k}(t) \notin \mathcal{C}^{i,j}
\label{eq: collision avoidance constraint}
\end{equation}
where $\mathcal{C}^{i,j}$ is the inter-collision model, which is a compact convex set that satisfies 
$\mathcal{C}^{j,i} = -\mathcal{C}^{i,j} = \{-\textbf{x} \mid \textbf{x} \in \mathcal{C}^{i,j}\}$
to maintain symmetry between agents.
In this work, we adopt the ellipsoidal collision model to consider the downwash effect of the quadrotor:
\begin{equation}
    \mathcal{C}^{i,j} = \{\textbf{x} \in \mathbb{R}^3 \mid \|E \textbf{x}\| \leq r^{i}+r^{j}\}
\label{eq: collision model}
\end{equation}
where $E=diag([1,1,1/c_{dw}])$ and $c_{dw}$ is the downwash coefficient.
Since (\ref{eq: collision avoidance constraint}) is a non-convex constraint, we linearize it using the convex hull property of Bernstein polynomial:
\begin{equation}
    \mathcal{H}^{i,j}_{k,m} \cap \mathcal{C}^{i,j} = \emptyset, \:\:\: \forall j \in \mathcal{I} \backslash i, m=1,\cdots,M
\label{eq: linearized collision avoidance constraint1}
\end{equation}
\begin{equation}
    \mathcal{H}^{i,j}_{k,m} = \text{Conv}(\{\textbf{c}^{i}_{k,m,l} - \textbf{c}^{j}_{k,m,l} \mid l=0,\cdots,n\})
\label{eq: linearized collision avoidance constraint2}
\end{equation}
where $\mathcal{H}^{i,j}_{k,m}$ is the convex hull of the control points of relative trajectory between the agents $i$ and $j$, $\mathcal{I}$ is a set of agents, and $\text{Conv}(\cdot)$ is the convex hull operator that returns a convex hull of the input set. 

\begin{theorem}
If the trajectories of all agents satisfy the condition (\ref{eq: linearized collision avoidance constraint1}), then there is no collision between agents.
\label{theorem: linearized collision avoidance constraint}
\end{theorem}

\begin{proof}
For any pair of the agents $i$ and $j$, the relative trajectory between two agents is a piecewise Bernstein polynomial:
\begin{equation}
    \textbf{p}^{i}_{k,m}(t) - \textbf{p}^{j}_{k,m}(t) = \sum_{l=0}^{n}(\textbf{c}^{i}_{k,m,l}-\textbf{c}^{j}_{k,m,l})b_{l,n}(\tau(t,k,m))
\label{eq: linearized collision avoidance constraint proof1}
\end{equation}
Due to the convex hull property of the Bernstein polynomial and (\ref{eq: linearized collision avoidance constraint1}), we obtain the following for $\forall m, t \in [T_{k+m-1}, T_{k+m}]$.
\begin{equation}
    \textbf{p}^{i}_{k,m}(t) - \textbf{p}^{j}_{k,m}(t) \in \mathcal{H}^{i,j}_{k,m}
\label{eq: linearized collision avoidance constraint proof2}
\end{equation}
\begin{equation}
    \textbf{p}^{i}_{k,m}(t) - \textbf{p}^{j}_{k,m}(t) \notin \mathcal{C}^{i,j}
\label{eq: linearized collision avoidance constraint proof3}
\end{equation}
Thus, there is no collision between agents because they satisfy the collision constraint (\ref{eq: collision avoidance constraint}) for any segment $m$.
\end{proof}

\section{Algorithm}
\label{sec: algorithm}
In this section, we introduce the trajectory planning algorithm that guarantees to return the feasible solution.
Alg. \ref{alg: trajectory planning algorithm} summarizes the overall process.
First, we receive the current positions, goal points and trajectories at the previous step for all agents as inputs.
Then, we plan initial trajectories using the trajectory at the previous replanning step (line 2, Sec. \ref{subsec: initial trajectory planning}).
Next, we construct collision avoidance constraints based on the initial trajectories (lines 4-7, Sec. \ref{subsec: safe flight corridor construction}, \ref{subsec: relative safe flight corridor construction}).
After that, we determine the intermediate goal using the grid-based path planner (line 8, Sec. \ref{subsec: goal planning}).
Finally, we obtain the safe trajectory by solving the quadratic programming (QP) problem (line 9, Sec. \ref{subsec: trajectory optimization}).
We will prove the feasibility of the optimization problem by showing that the initial trajectory satisfies all the constraints.

\setlength{\textfloatsep}{5pt}
\begin{algorithm}
\SetAlgoLined
\KwIn{current positions $\textbf{p}^{i \in \mathcal{I}}_{k}(T_k)$, goal points $\textbf{g}^{i \in \mathcal{I}}$, trajectories at the previous step $\textbf{p}^{i \in \mathcal{I}}_{k-1}(t)$, 
and 3D occupancy map $\mathcal{W}$}
\KwOut{trajectory $\textbf{p}^{i}_{k}(t)$ for agent $i$, $t \in [T_{k},T_{k+M}]$}
  \For{$\forall i \in \mathcal{I} $}{
    $\hat{\textbf{p}}^{i}_{k}(t) \gets$ planInitialTraj($\textbf{p}^{i}_{k}(T_{k}), \textbf{p}^{i}_{k-1}(t)$)\;
  }
  $\mathcal{S}^{i}_{k} \gets$ buildSFC($\hat{\textbf{p}}^{i}_{k}(t), \mathcal{W}$)\;
  \For{$\forall j \in \mathcal{I} \backslash i$}{
    $\mathcal{L}^{i,j}_{k} \gets$ buildLSC($\hat{\textbf{p}}^{i}_{k}(t), \hat{\textbf{p}}^{j}_{k}(t)$)\;
  }
  $\textbf{g}^{i}_{curr} \gets$ planCurrentGoal($\hat{\textbf{p}}^{i \in \mathcal{I}}_{k}(t), \textbf{g}^{i \in \mathcal{I}}, \mathcal{W}$)\;
  $\textbf{p}^{i}_{k}(t) \gets$ trajOpt($ \mathcal{S}^{i}_{k}, \mathcal{L}^{i,\forall j \in \mathcal{I} \backslash i}_{k}, \textbf{g}^{i}_{curr})$\;
  \KwRet{$\textbf{p}^{i}_{k}(t)$}
\caption{Trajectory planning for agent $i$ at replanning step $k$}
\label{alg: trajectory planning algorithm}
\end{algorithm}

\subsection{Initial Trajectory Planning}
\label{subsec: initial trajectory planning}
We utilize the trajectory at the previous replanning step, $\textbf{p}^{i}_{k-1}(t)$, as an initial trajectory. If it is the first step of the planning, we then use the current position instead. Suppose that the replanning period is equal to the segment time $\Delta t$. Then we can represent the initial trajectory as piecewise Bernstein polynomial.
\begin{equation}
\begin{alignedat}{2}
    \hat{\textbf{p}}^{i}_{k}(t) = 
    \begin{cases} 
    \ \textbf{p}^{i}_{0}(T_{0})   & k = 0, t \in [T_{0}, T_{M}] \\  
    \ \textbf{p}^{i}_{k-1,m+1}(t)   & k > 0, t \in [T_{k}, T_{k+M-1}] \\    
    \ \textbf{p}^{i}_{k-1,M}(T_{k+M-1})   & k > 0, t \in [T_{k+M-1}, T_{k+M}] \\
    \end{cases}
\end{alignedat}
\label{eq: initial trajectory}
\end{equation}
where $\hat{\textbf{p}}^{i}_{k}(t)$ is the initial trajectory at the replanning step $k$.
We denote the $m^{th}$ segment of the initial trajectory as $\hat{\textbf{p}}^{i}_{k,m}(t)$ and the $l^{th}$ control point of this segment as $\hat{\textbf{c}}^{i}_{k,m,l}$.
Note that the control point of initial trajectory can be represented using the previous ones (e.g. $\hat{\textbf{c}}^{i}_{k,m,l} = \textbf{c}^{i}_{k-1,m+1,l}$ for $\forall m < M,l$) and there is no collision between initial trajectories if there is no collision between the previous ones. 

\subsection{Safe Flight Corridor Construction}
\label{subsec: safe flight corridor construction}
A safe flight corridor (SFC), $\mathcal{S}^{i}_{k,m}$, is a convex set that prevents the collision with the static obstacles:
\begin{equation}
  \mathcal{S}^{i}_{k,m} \oplus \mathcal{C}^{i,o} \subset \mathcal{F}
\label{eq: safe flight corridor definition2}
\end{equation}
To guarantee the feasibility of the constraints, we construct the SFC using the initial trajectory and the SFC of the previous step.
\begin{equation}
\begin{alignedat}{2}
    \mathcal{S}^{i}_{k,m} = 
    \begin{cases} 
    \ \mathcal{S}(\hat{\textbf{c}}^{i}_{k,M,n})   & k = 0 \text{ or } m = M \\  
    \ \mathcal{S}^{i}_{k-1,m+1}   & \text{else} \\
    \end{cases}
\end{alignedat}
\label{eq: safe flight corridor construction}
\end{equation}
where $\mathcal{S}(\hat{\textbf{c}}^{i}_{k,M,n})$ is a convex set that contains the point  $\hat{\textbf{c}}^{i}_{k,M,n}$ and satisfies $\mathcal{S}(\hat{\textbf{c}}^{i}_{k,M,n}) \oplus \mathcal{C}^{i,o} \subset \mathcal{F}$. Note that the SFC in (\ref{eq: safe flight corridor construction}) always exists for all replanning steps if the control points of the initial trajectory do not collide with static obstacles. We can find this using the axis-search method \cite{jungwon2020efficient}.

Assume that the control points of $\textbf{p}^{i}_{k,m}$ is confined in the corresponding SFC, i.e. for $m=1,\cdots,M$ and $l=0,\cdots,n$:
\begin{equation}
  \textbf{c}^{i}_{k,m,l} \in \mathcal{S}^{i}_{k,m}
\label{eq: safe flight corridor previous trajectory}
\end{equation}
Then, there is no collision with static obstacles due to the convex hull property of the Bernstein polynomial \cite{tang2016safe}.
Also, we can prove that the control points of the initial trajectory always belong to the corresponding SFC for all replanning steps.

\begin{lemma}
(Feasibility of SFC) Assume that $\textbf{c}^{i}_{k-1,m,l} \in \mathcal{S}^{i}_{k-1,m}$ for $\forall m,l$ at the replanning step $k>0$. Then, $\hat{\textbf{c}}^{i}_{k,m,l}(t) \in \mathcal{S}^{i}_{k,m}$ for $\forall k, m, l$.
\label{lemma: feasibility of safe flight corridor}
\end{lemma}

\begin{proof}
If $k=0$, we obtain $\hat{\textbf{c}}^{i}_{0,m,l} \in \mathcal{S}^{i}_{0,m}$ for $\forall m, l$ since $\hat{\textbf{c}}^{i}_{0,m,l} = \hat{\textbf{c}}^{i}_{0,M,n}$ for $\forall m, l$.
If $k>0$ and $m<M$, we have $\hat{\textbf{c}}^{i}_{k,m,l} = \textbf{c}^{i}_{k-1,m+1,l} \in \mathcal{S}^{i}_{k-1,m+1} = \mathcal{S}^{i}_{k,m}$ for $\forall l$ due to (\ref{eq: initial trajectory}) and (\ref{eq: safe flight corridor previous trajectory}).
If $k>0$ and $m=M$, we obtain $\hat{\textbf{c}}^{i}_{k,M,l} = \hat{\textbf{c}}^{i}_{k,M,n} \in \mathcal{S}(\hat{\textbf{c}}^{i}_{k,M,n}) = \mathcal{S}^{i}_{k,M}$ for $\forall l$.
Therefore, we have $\hat{\textbf{c}}^{i}_{k,m,l}(t) \in \mathcal{S}^{i}_{k,m}$ for $\forall k,m,l$.
\end{proof}


\subsection{Linear Safe Corridor Construction}
\label{subsec: relative safe flight corridor construction}
We define the LSC as a linear constraint that satisfies the following conditions:
\begin{equation}
  \mathcal{L}^{i,j}_{k,m,l} = \{\textbf{x} \in \mathbb{R}^3 \mid (\textbf{x} - \hat{\textbf{c}}^{j}_{k,m,l}) \cdot \textbf{n}^{i,j}_{m} - d^{i,j}_{m,l} > 0\}
\label{eq: relative safe flight corridor definition2}
\end{equation}
\begin{equation}
  \textbf{n}^{i,j}_{m} = -\textbf{n}^{j,i}_{m}
\label{eq: relative safe flight corridor definition3}
\end{equation}
\begin{equation}
  d^{i,j}_{m,l} + d^{j,i}_{m,l} \geq \max \langle \mathcal{C}^{i,j}, \textbf{n}^{i,j}_{m} \rangle + (\hat{\textbf{c}}^{i}_{k,m,l} - \hat{\textbf{c}}^{j}_{k,m,l}) \cdot \textbf{n}^{i,j}_{m}
\label{eq: relative safe flight corridor definition4}
\end{equation}
where $\textbf{n}^{i,j}_{m} \in \mathbb{R}^{3}$ is the normal vector, $d^{i,j}_{m,l}$ is the safety margin and $\max \langle \mathcal{C}^{i,j}, \textbf{n}^{i,j}_{m} \rangle = \max_{\textbf{x} \in \mathcal{C}^{i,j}} \textbf{x} \cdot \textbf{n}^{i,j}_{m}$. Let $\hat{\mathcal{H}}^{i,j}_{k,m}$ is the convex hull of the control points of relative initial trajectory, i.e.  $\hat{\mathcal{H}}^{i,j}_{k,m}=\text{Conv}(\{\hat{\textbf{c}}^{i}_{k,m,l} - \hat{\textbf{c}}^{j}_{k,m,l} \mid l=0,\cdots,n\})$. Lemmas \ref{lemma: safety of relative safe flight corridor} and \ref{lemma: existence and feasibility of relative safe flight corridor} present the main properties of LSC.

\begin{lemma}
(Safety of LSC) If $\textbf{c}^{i}_{k,m,l} \in \mathcal{L}^{i,j}_{k,m,l}$, $\textbf{c}^{j}_{k,m,l} \in \mathcal{L}^{j,i}_{k,m,l}$ for $\forall m,l$ then we obtain $\mathcal{H}^{i,j}_{k,m} \cap  \mathcal{C}^{i,j} = \emptyset$ which implies that the agent $i$ does not collide with the agent $j$.
\label{lemma: safety of relative safe flight corridor}
\end{lemma}

\begin{proof}
We obtain the following by adding the inequality in (\ref{eq: relative safe flight corridor definition2}) for each agent $i$ and $j$:

\begin{equation}
  (\textbf{c}^{i}_{k,m,l}-\hat{\textbf{c}}^{i}_{k,m,l}) \cdot \textbf{n}^{i,j}_{m} + (\textbf{c}^{j}_{k,m,l}-\hat{\textbf{c}}^{j}_{k,m,l}) \cdot \textbf{n}^{j,i}_{m} 
  - (d^{i,j}_{m,l} + d^{j,i}_{m,l}) > 0
\label{eq: relative safe flight corridor safety0}
\end{equation}

This can be simplified using (\ref{eq: relative safe flight corridor definition3}) and (\ref{eq: relative safe flight corridor definition4}):
\begin{equation}
  (\textbf{c}^{i}_{k,m,l}-\textbf{c}^{j}_{k,m,l}) \cdot \textbf{n}^{i,j}_{m} + (\hat{\textbf{c}}^{i}_{k,m,l}-\hat{\textbf{c}}^{j}_{k,m,l}) \cdot \textbf{n}^{i,j}_{m} 
  - (d^{i,j}_{m,l} + d^{j,i}_{m,l}) > 0
\label{eq: relative safe flight corridor safety1}
\end{equation}
\begin{equation}
  (\textbf{c}^{i}_{k,m,l}-\textbf{c}^{j}_{k,m,l}) \cdot \textbf{n}^{i,j}_{m} > \max \langle \mathcal{C}^{i,j}, \textbf{n}^{i,j}_{m} \rangle
\label{eq: relative safe flight corridor safety2}
\end{equation}
The above inequality satisfies for $\forall l$, thus we have the following for any $\lambda_{\forall l} \geq 0$ s.t. $\sum_{l=0}^{n}\lambda_{l}=1$:
\begin{equation}
  \sum_{l=0}^{n}\lambda_{l}(\textbf{c}^{i}_{k,m,l}-\textbf{c}^{j}_{k,m,l}) \cdot \textbf{n}^{i,j}_{m} > \max \langle \mathcal{C}^{i,j}, \textbf{n}^{i,j}_{m} \rangle
\label{eq: relative safe flight corridor safety3}
\end{equation}
\begin{equation}
  \mathcal{H}^{i,j}_{k,m} \cap  \mathcal{C}^{i,j} = \emptyset
\label{eq: relative safe flight corridor safety4}
\end{equation}
Therefore, there is no collision between the agents $i$ and $j$ by Thm. \ref{theorem: linearized collision avoidance constraint}.
\end{proof}

\begin{lemma}
(Existence and Feasibility of LSC) If $\hat{\mathcal{H}}^{i,j}_{k,m} \cap \mathcal{C}^{i,j} = \emptyset$ for $\forall i,j \in \mathcal{I}, m, l$, then there exists $\mathcal{L}^{i,j}_{k,m,l}$ that satisfies the definition of LSC and $\hat{\textbf{c}}^{i}_{k,m,l} \in \mathcal{L}^{i,j}_{k,m,l}$ for $\forall i,j \in \mathcal{I}, m, l$.
\label{lemma: existence and feasibility of relative safe flight corridor}
\end{lemma}

\begin{proof}
$\hat{\mathcal{H}}^{i,j}_{k,m}$ and $\mathcal{C}^{i,j}$ are disjoint compact convex sets. Thus, by the hyperplane seperation theorem \cite{boyd2004convex}, there exists $\textbf{n}_{s}$ such that:
\begin{equation}
  \min \langle \hat{\mathcal{H}}^{i,j}_{k,m}, \textbf{n}_{s} \rangle > \max \langle \mathcal{C}^{i,j}, \textbf{n}_{s} \rangle
\label{eq: relative safe flight corridor existence1}
\end{equation}
where $\min \langle \hat{\mathcal{H}}^{i,j}_{k,m}, \textbf{n}_{s} \rangle = \min_{\textbf{x} \in \hat{\mathcal{H}}^{i,j}_{k,m}} \textbf{x} \cdot \textbf{n}_{s}$.
Here, we set the normal vector and the safety margin of $\mathcal{L}^{i,j}_{k,m,l}$ as follows:
\begin{equation}
  \textbf{n}^{i,j}_{m} = -\textbf{n}^{j,i}_{m} = \textbf{n}_{s}
\label{eq: relative safe flight corridor existence2}
\end{equation}
\begin{equation}
  d^{i,j}_{m,l} = d^{j,i}_{m,l} = \frac{1}{2}(\max \langle \mathcal{C}^{i,j}, \textbf{n}^{i,j}_{m} \rangle + (\hat{\textbf{c}}^{i}_{k,m,l} - \hat{\textbf{c}}^{j}_{k,m,l}) \cdot \textbf{n}^{i,j}_{m})
\label{eq: relative safe flight corridor existence3}
\end{equation}
We can observe that they fulfill the definition of LSC. Furthermore, we have $(\hat{\textbf{c}}^{i}_{k,m,l} - \hat{\textbf{c}}^{j}_{k,m,l}) \cdot \textbf{n}^{i,j}_{m} - d^{i,j}_{m,l} = \frac{1}{2}((\hat{\textbf{c}}^{i}_{k,m,l} - \hat{\textbf{c}}^{j}_{k,m,l}) \cdot \textbf{n}^{i,j}_{m} - \max \langle \mathcal{C}^{i,j}, \textbf{n}^{i,j}_{m} \rangle) > 0$ due to (\ref{eq: relative safe flight corridor existence1}).
It indicates that $\hat{\textbf{c}}^{i}_{k,m,l} \in \mathcal{L}^{i,j}_{k,m,l}$ for $\forall i,j \in \mathcal{I}, m, l$.
\end{proof}

\begin{figure}[t]
\centering
\includegraphics[width = 0.8\linewidth]{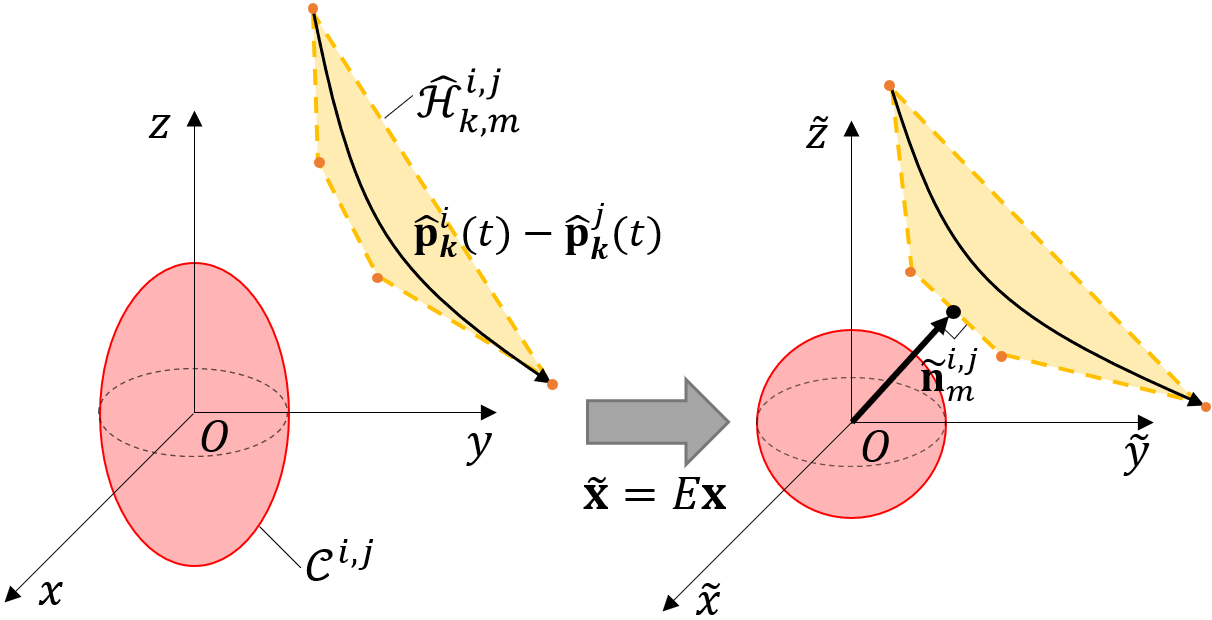}
\caption{
  Decision method of the normal vector of LSC. The red ellipsoid is the collision model between two agents, and the orange-shaded region is the convex hull that consists of the control points of the relative initial trajectory. 
}
\label{fig: LSC construction}
\end{figure}

We construct the LSC based on Lemma \ref{lemma: existence and feasibility of relative safe flight corridor}.
First, we perform the coordinate transformation to convert the collision model to the sphere as shown in Fig. \ref{fig: LSC construction}. 
Next, we conduct the GJK algorithm \cite{gilbert1988fast} to find the closest points between $\hat{\mathcal{H}}^{i,j}_{k,m}$ and the collision model.  
Then, we determine the normal vector utilizing the vector crossing the closest points and compute the safety margin using (\ref{eq: relative safe flight corridor existence3}).
Finally, we obtain the LSC by reversing the coordinate transform.

\begin{figure}[t]
    \centering
    \begin{subfigure}[t]{0.21\textwidth}
        \includegraphics[width=\textwidth]{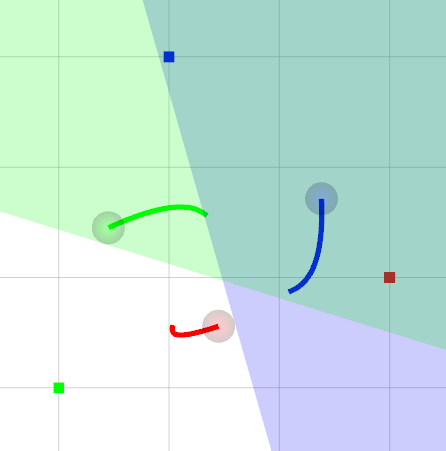}
        \caption{BVC \cite{zhou2017fast}}
        \label{fig: comparision bvc}
    \end{subfigure}
    ~ 
    \begin{subfigure}[t]{0.21\textwidth}
        \includegraphics[width=\textwidth]{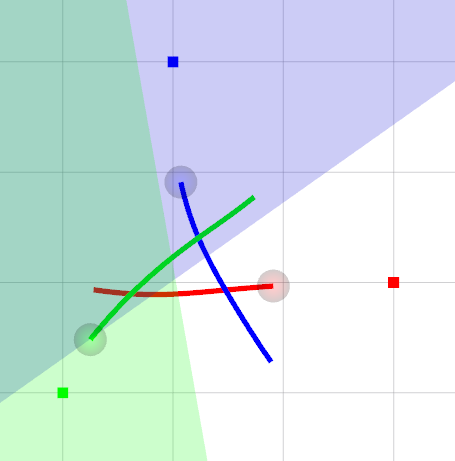}
        \caption{LSC}
        \label{fig: comparision lsc}
    \end{subfigure}
    \caption{
    Trajectory planning comparison between BVC \cite{zhou2017fast} and LSC. The colored lines, small squares and circles are desired trajectories, goal points and desired position at the end of planning horizon (i.e. $\textbf{p}^{i}_{k}(T_{k+M})$) respectively. The color-shaded regions denote the collision constraints for the red agent at the end of planning horizon.
    }
    \label{fig: comparision bvc vs lsc}
\end{figure}

We compare the LSC with the BVC \cite{zhou2017fast} in Fig. \ref{fig: comparision bvc vs lsc}. Since the BVC is generated using the agent's current position only, the desired trajectories from the BVC remain within each static cell, which leads to a conservative maneuver as shown in Fig. \ref{fig: comparision bvc}. 
 On the other hand, we construct the LSC utilizing the full trajectory at the previous step.
Thus, we can obtain more aggressive maneuver while ensuring collision avoidance, as depicted in Fig. \ref{fig: comparision lsc}.   


\begin{algorithm}
\SetAlgoLined
\KwIn{initial trajectories $\hat{\textbf{p}}^{i \in \mathcal{I}}_{k}(t)$, goal points $\textbf{g}^{i \in \mathcal{I}}$ and 3D occupancy map $\mathcal{W}$}
\KwOut{current goal point $\textbf{g}^{i}_{curr}$}
  $\mathcal{P} \gets \emptyset$\;
  \For{$\forall j \in \mathcal{I} \backslash i$}{
    \If{the agent $j$ satisfies (\ref{eq: goal planning condition1}), (\ref{eq: goal planning condition2}), (\ref{eq: goal planning condition3}) or $\|\hat{\textbf{p}}^{i}_{k}(T_{k}) - \textbf{g}^{i}\| < d_{g}$}{
      $\mathcal{P} = \mathcal{P} \cup \{j\}$\;
    }
  }
  $q \gets $ findClosestAgent($\mathcal{P}$)\;
  \eIf{$\|\hat{\textbf{p}}^{i}_{k}(T_{k}) - \hat{\textbf{p}}^{q}_{k}(T_{k}) \| < d_{th}$}{
     $\textbf{g}^{i}_{curr} \gets \hat{\textbf{p}}^{q}_{k}(T_{k}) + \frac{\hat{\textbf{p}}^{i}_{k}(T_{k}) - \hat{\textbf{p}}^{q}_{k}(T_{k})}{\|\hat{\textbf{p}}^{i}_{k}(T_{k}) - \hat{\textbf{p}}^{q}_{k}(T_{k})\|} d_{rep}$ 
  }
  {
    $\mathcal{D} \gets$ gridBasedPlanner($\hat{\textbf{p}}^{i}_{k}(T_{k}), \mathcal{P}, \mathcal{W}$)\;
    $\textbf{g}^{i}_{curr} \gets$ findLOSFreeGoal($\hat{\textbf{p}}^{i}_{k}(T_{k}), \mathcal{D}, \mathcal{P}, \mathcal{W}$)\;
  }
  \KwRet{$\textbf{g}^{i}_{curr}$}
\caption{planCurrentGoal}
\label{alg: goal planning algorithm}
\end{algorithm}

\subsection{Goal Planning}
\label{subsec: goal planning}
In this work, we adopt the priority-based goal planning to prevent the deadlock in cluttered environments.
Alg. \ref{alg: goal planning algorithm} describes the proposed goal planning method.
We give the higher priority to the agent that has a smaller distance to the goal (\ref{eq: goal planning condition1}). 
However, we assign the lowest priority to the goal-reached agent to prevent them from blocking other agent's path (\ref{eq: goal planning condition2}).
If an agent moves away from the current agent, then we do not consider the priority of that agent for smooth traffic flow (\ref{eq: goal planning condition3})  (lines 1-6).
\begin{equation}
  \|\hat{\textbf{p}}^{j}_{k}(T_{k}) - \textbf{g}^{j}\| < \|\hat{\textbf{p}}^{i}_{k}(T_{k}) - \textbf{g}^{i}\|
\label{eq: goal planning condition1}
\end{equation}
\begin{equation}
  \|\hat{\textbf{p}}^{j}_{k}(T_{k}) - \textbf{g}^{j}\| > d_{g}
\label{eq: goal planning condition2}
\end{equation}
\begin{equation}
  (\hat{\textbf{p}}^{j}_{k}(T_{k+M}) - \hat{\textbf{p}}^{j}_{k}(T_{k})) \cdot (\hat{\textbf{p}}^{i}_{k}(T_{k}) - \hat{\textbf{p}}^{j}_{k}(T_{k})) > 0
\label{eq: goal planning condition3}
\end{equation}
Next, we find the nearest agent among the higher priority agents. If the distance to the nearest agent is less than the threshold $d_{th}$, then we assign the goal to increase the distance to the closest agent (lines 7-9). It is necessary to prevent agents from clumping together. If there is no such agent, we search a discrete path $\mathcal{D}$ using a grid-based path planner. Here, the higher-priority agent is considered as a static obstacle. If the planner fails to find a solution, then we retry to find the path without considering the higher-priority agent (line 11).
Finally, we find the \textit{LOS-free} goal in the discrete path (line 12). The \textit{LOS-free} goal means that the line segment between the goal and the agent's current position does not collide with higher priority agents and static obstacles.

\begin{figure}[t]
    \centering
    \begin{subfigure}[t]{0.23\textwidth}
        \includegraphics[width=\textwidth]{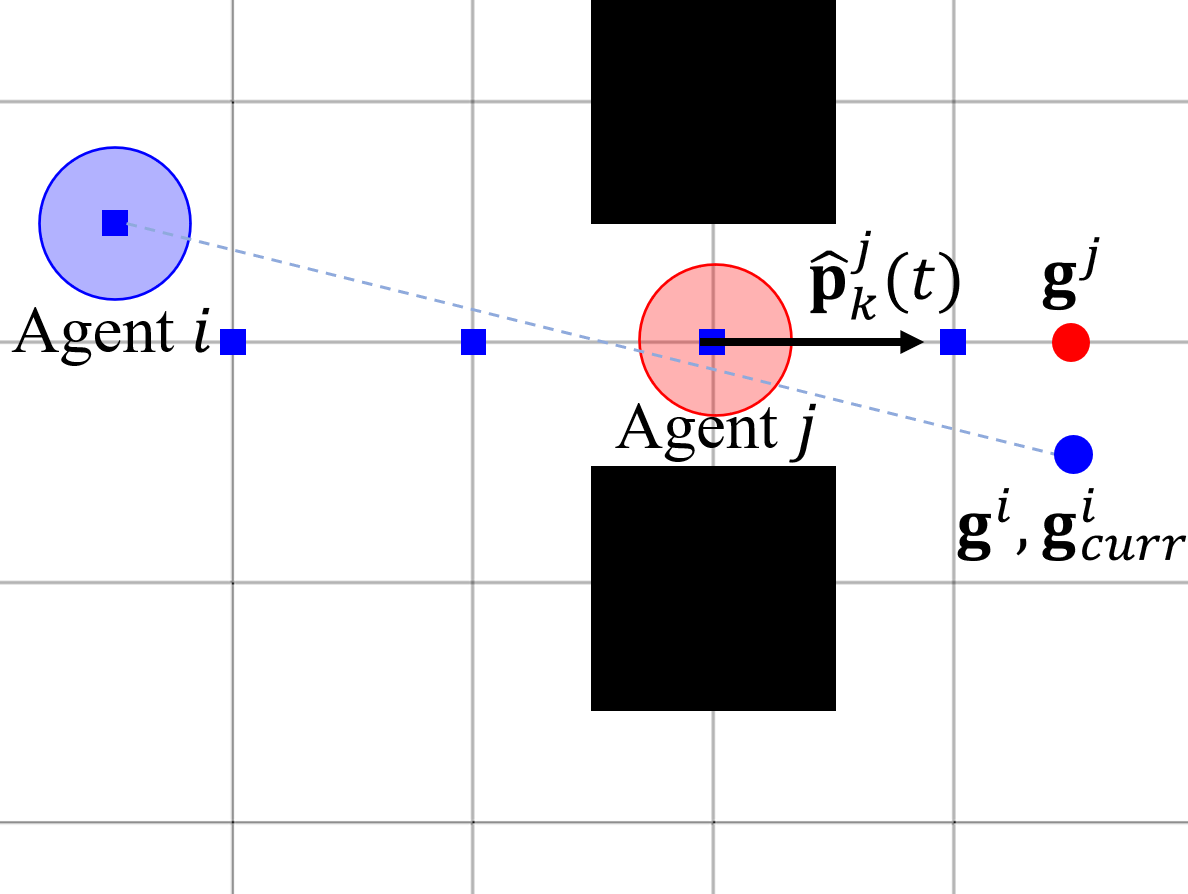}
        \label{fig: goal planning1}
    \end{subfigure}
    ~ 
    \begin{subfigure}[t]{0.23\textwidth}
        \includegraphics[width=\textwidth]{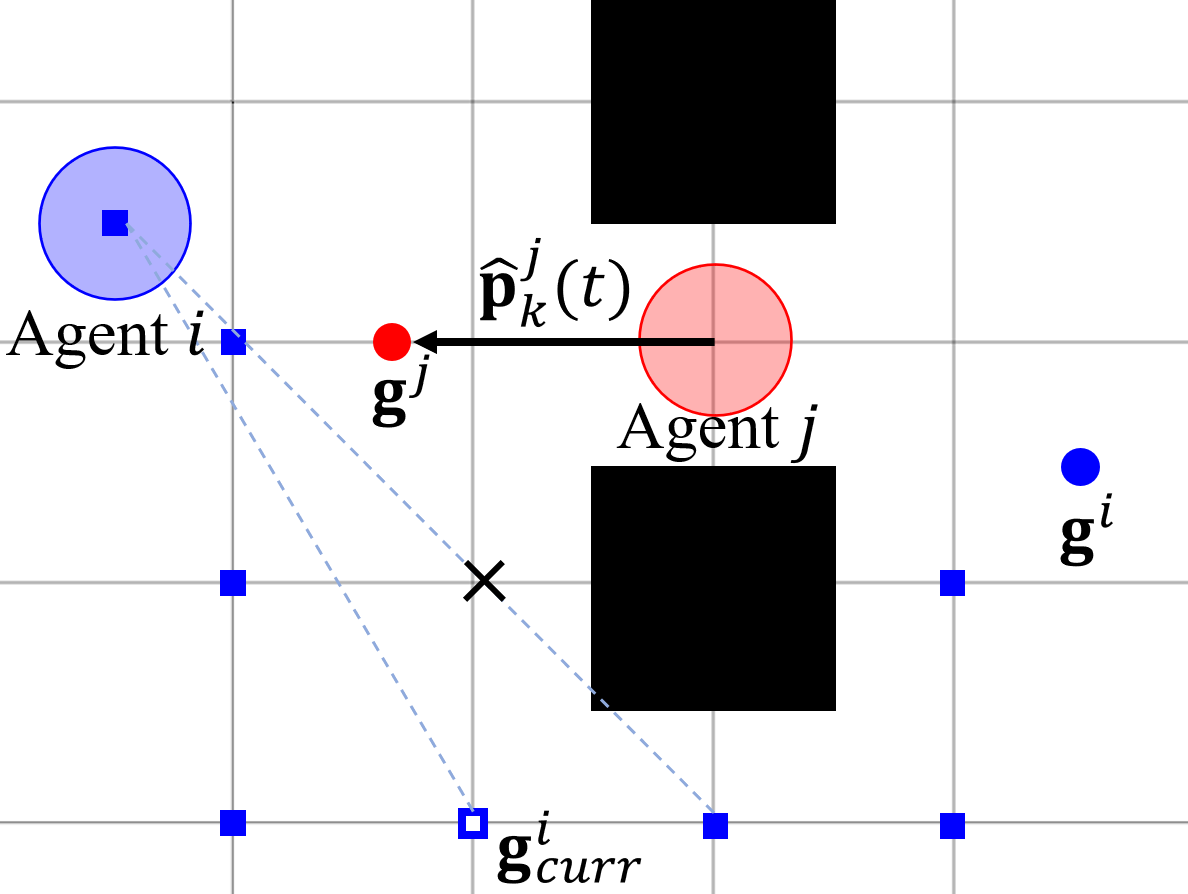}
        \label{fig: goal planning2}
    \end{subfigure}
    \caption{
    Goal planning depending on the direction of the initial trajectory. The blue and red circles are the current position of agents $i$ and $j$ respectively. The small circles are the goal points of each agent and blue small squares are the discrete path of agent $i$ and black squares are static obstacles.
    }
    \label{fig: goal planning}
\end{figure}

Fig. \ref{fig: goal planning} describes the mechanism of the Alg. \ref{alg: goal planning algorithm}.
In the left figure, the agent $j$ has lower priority because it is moving away from the agent $i$.
Due to that, the agent $i$ ignores the agent $j$ when planning the discrete path. 
On the other hand, the agent $j$ in the right figure has higher priority because it moves to the agent $i$ and has a smaller distance to the goal.
Thus, the agent $i$ finds the LOS-free goal in the discrete path that detours the agent $j$ and static obstacles.

\subsection{Trajectory Optimization}
\label{subsec: trajectory optimization}
We reformulate the distributed trajectory optimization problem as a quadratic programming (QP) problem using the SFC and LSC:
\begin{equation}
\begin{aligned}
& \underset{\textbf{c}^{i}_{k}}{\text{minimize}}     & & J^{i}_{err} + J^{i}_{der} \\
& \text{subject to}   & & A_{eq}\textbf{c}^{i}_{k} = \textbf{b}_{eq} \\
&                     & & A_{dyn}\textbf{c}^{i}_{k} \preceq \textbf{b}_{dyn} \\
&                     & & \textbf{c}^{i}_{k,m,l} \in \mathcal{S}^{i}_{k,m} & & \forall j \in \mathcal{I} \backslash i, m, l \\
&                     & & \textbf{c}^{i}_{k,m,l} \in \mathcal{L}^{i,j}_{k,m,l} & & \forall j \in \mathcal{I} \backslash i, m, l \\
&                     & & \textbf{c}^{i}_{k,M,0} = \textbf{c}^{i}_{k,M,l} & & \forall l>0
\end{aligned}
\label{eq: trajectory optimization}
\end{equation}
Here, we add the last constraint to guarantee dynamic feasibility of the initial trajectory. Note that the constraints in (\ref{eq: trajectory optimization}) guarantees collision-free due to the SFC and Lemma \ref{lemma: safety of relative safe flight corridor}.
Furthermore, Thm. \ref{theorem: feasibility of optimization problem} shows that this problem consists of feasible constraints.
Therefore, we can obtain a feasible solution from Alg. \ref{alg: trajectory planning algorithm} for all replanning steps by using a conventional convex solver.

\begin{theorem}
Suppose that there is no collision in the initial configuration and the replanning period is equal to the segment time $\Delta t$. Then, the solution of (\ref{eq: trajectory optimization}) always exists for all replanning steps.
\label{theorem: feasibility of optimization problem}
\end{theorem}

\begin{proof}
If $k=0$, then we obtain $\hat{\mathcal{H}}^{i,j}_{0,m} \cap \mathcal{C}^{i,j} = \emptyset$ for $\forall j \in \mathcal{I} \backslash i,m$ by the assumption. Hence the control points of the initial trajectory $\hat{\textbf{c}}^{i}_{0}$ are the solution to (\ref{eq: trajectory optimization}) by Lemmas \ref{lemma: feasibility of safe flight corridor}, \ref{lemma: existence and feasibility of relative safe flight corridor}.

If there exists a solution at the previous replanning step $k-1$, then the initial trajectory satisfies $A_{eq}\hat{\textbf{c}}^{i}_{k} = \textbf{b}_{eq}$, $A_{dyn}\hat{\textbf{c}}^{i}_{k} \preceq \textbf{b}_{dyn}$, $\hat{\textbf{c}}^{i}_{k,m,l} \in \mathcal{S}^{i}_{k,m}$ and $\hat{\textbf{c}}^{i}_{k,M,0} = \hat{\textbf{c}}^{i}_{k,M,l}$ for $\forall j \in \mathcal{I} \backslash i, m,l$ due to (\ref{eq: initial trajectory}) and Lemma \ref{lemma: feasibility of safe flight corridor}.
Also, we have $\mathcal{H}^{i,j}_{k-1,m} \cap  \mathcal{C}^{i,j} = \emptyset$ for $\forall j \in \mathcal{I} \backslash i, m$ by Lemma \ref{lemma: safety of relative safe flight corridor}.
Here, we can derive that $\hat{\mathcal{H}}^{i,j}_{k,m} \cap \mathcal{C}^{i,j} = \emptyset$ for $\forall j \in \mathcal{I} \backslash i, m$ using the fact that $\hat{\mathcal{H}}^{i,j}_{k,m} = \mathcal{H}^{i,j}_{k-1,m+1}$ for $\forall m < M$ and $\hat{\mathcal{H}}^{i,j}_{k,M} = \{\textbf{c}^{i}_{k-1,M,n}\} \in \mathcal{H}^{i,j}_{k-1,M}$. Thus, we obtain $\hat{\textbf{c}}^{i}_{k,m,l} \in \mathcal{L}^{i,j}_{k,m,l}$ by Lemma \ref{lemma: existence and feasibility of relative safe flight corridor}. To summarize, the control points of the initial trajectory $\hat{\textbf{c}}^{i}_{k}$ satisfies all constraints of (\ref{eq: trajectory optimization}).
Therefore, the solution of (\ref{eq: trajectory optimization}) always exists for all replanning steps by the mathematical induction.
\end{proof}

If we conduct the mission with actual quadrotors, the assumption used in Thm. \ref{theorem: feasibility of optimization problem} may not be satisfied due to external disturbance or tracking error. To solve this problem, we adopt \textit{event-triggered replanning} proposed in \cite{luis2020online}. This method ignores the tracking error when the error is small. In other word, it utilizes the desired state from the previous trajectory ($\textbf{p}^{i}_{k-1}(T_{k})$) as the initial condition of the current optimization problem (\ref{eq: equality constraints}).
If the detected disturbance is too large, we replace the planner with our previous work \cite{park2020online} until the error vanished.

\section{EXPERIMENTS}
\label{sec: experiments}
All simulation and experiments were conducted on a laptop running Ubuntu 18.04. with Intel Core i7-9750H @ 2.60GHz CPU and 32G RAM. All methods used for comparison were implemented in C++ except for DMPC tested in MATLAB R2020a.
We modeled the quadrotor with radius $r^{\forall i} = 0.15$ m, downwash coefficient $c_{dw} = 2$, maximum velocity $[1.0, 1.0, 1.0]$ m/s, maximum acceleration $[2.0, 2.0, 2.0]$ m/$\text{s}^2$ based on the experimental result with Crazyflie 2.0.
We use the Octomap \cite{hornung2013octomap} to represent the obstacle environment.
We formulated the piecewise Bernstein polynomial with the degree of polynomials $n=5$, the number of segments $M=5$ and the segment time $\Delta t = 0.2$ s, where the total time horizon is 1 s.
We set the parameters for the objective function and Alg. \ref{alg: goal planning algorithm} to $w_{e,m} = 1$, $w_{d,3}=0.01$, $d_{g} = 0.1$ m, $d_{th} = 0.4$ m and $d_{rep} = 0.5$ m.
We used the openGJK package \cite{montanari2018opengjk} for GJK algorithm and the CPLEX QP solver \cite{cplex201612} for trajectory optimization.
When we execute the proposed algorithm, we set the interval between trajectory update to $\Delta t = 0.2$ s to match the assumption in Thm. \ref{theorem: feasibility of optimization problem}. It means that all agents begin to plan the trajectories at the same time and update it with the rate of 5 Hz.

\begin{figure}[t]
    \centering
    \includegraphics[width = 0.48\textwidth]{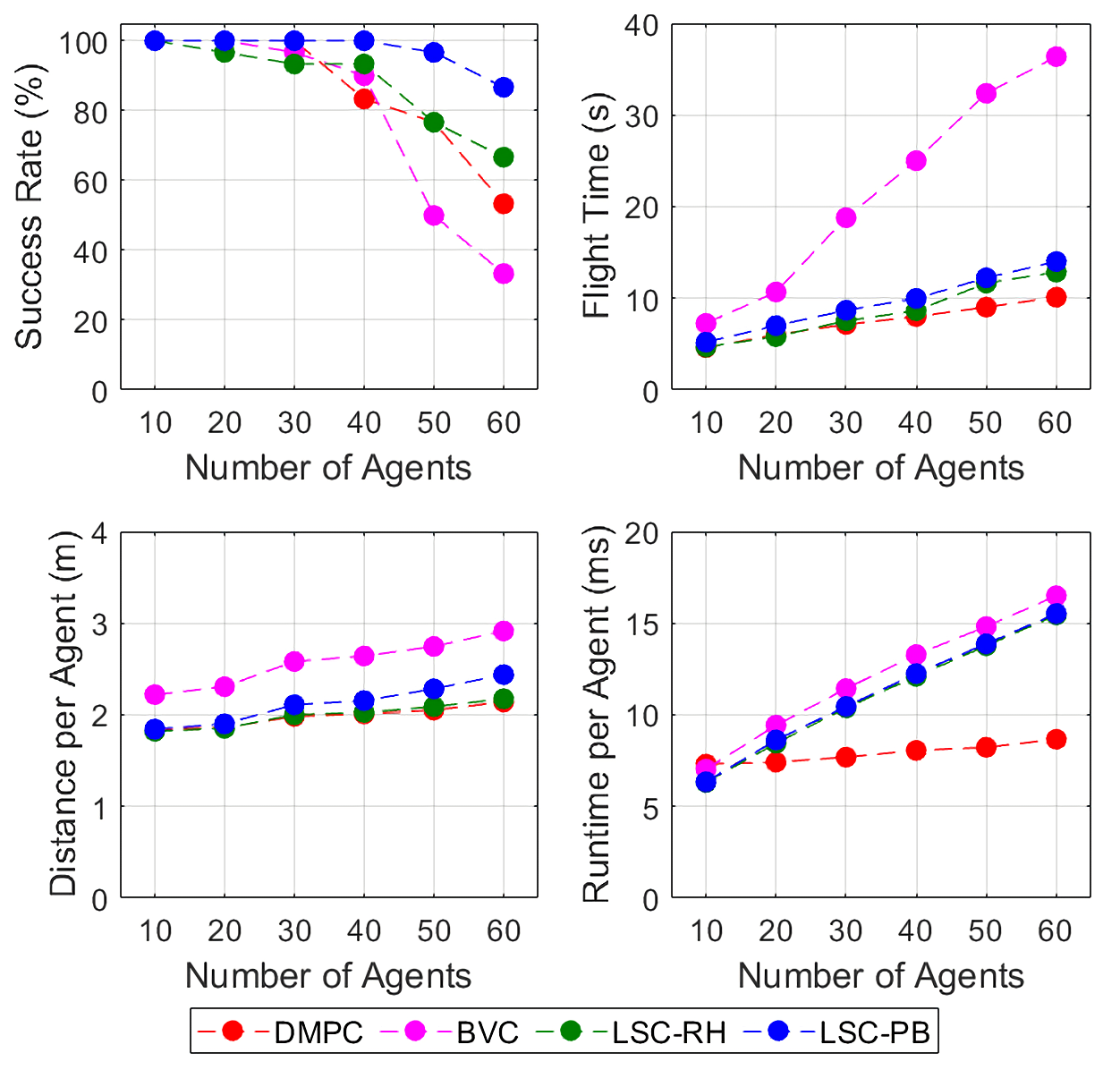}
    \caption{
        Performance comparison in a 3 m $\times$ 3 m $\times$ 2 m empty space. We averaged the value from success cases among 30 different random trials (shown best in color).
    }
    \label{fig:simulation1}
\end{figure}

\subsection{Simulation in Obstacle-free Space}
To validate the performance of the proposed method, we compare four online trajectory planning algorithms at the obstacle-free space:
1) DMPC (input space) \cite{luis2020online},
2) BVC \cite{zhou2017fast} with right-hand rule,
3) LSC with right-hand rule (LSC-RH),
4) LSC with priority-based goal planning (LSC-PB).
We conducted the simulation with 10 to 60 quadrotors in a 3 m $\times$ 3 m $\times$ 2 m empty space. We executed 30 missions with randomly allocated start and goal points for each number of agents.
When judging the collision, we use a smaller model with $r^{\forall i} = 0.1$ m and $c_{dw} = 2.25$ for DMPC because it uses soft constraints, while the other methods use the original one.

The top-left graph of Fig. \ref{fig:simulation1} describes the success rate of each method.
We observed that the collision occurred in all failure cases of DMPC even though we adopted the smaller collision model. It implies that collision constraint including the slack variable cannot guarantees the safety in a dense environment even with the bigger safety margin.
On the contrary, there was no collision when we use BVC and LSC-based methods, and the deadlock was the only reason for failure. It indicates that the proposed method does not cause optimization failure as guaranteed by Thm. \ref{theorem: feasibility of optimization problem}.
Among the four algorithms, the LSC with priority-based goal planning method shows the highest success rate for all cases. 
This result shows that priority-based goal planning can solve deadlock better than the right-hand rule.  

The top-right and bottom-left figures of Fig. \ref{fig:simulation1} show the total flight time and flight distance per agent, respectively.
The proposed method has $50\%$ less flight time and $17 \%$ less flight distance on average compared to the BVC-based method. Such performance is similar to DMPC, but the proposed algorithm still guarantees collision avoidance.
Compared to the right-hand rule, the priority-based goal planning costs $11\%$ more flight time and $6\%$ more flight distance, but it is acceptable considering the increase in success rate. 

The bottom-right graph of Fig. \ref{fig:simulation1} describes the average computation time per agent.
The proposed method takes 10.5 ms for 30 agents and 15.5 ms for 60 agents.
The proposed algorithm consumes similar computation time to BVC, but needs more computation time than DMPC because it involves more constraints to guarantee collision avoidance.

\begin{figure}
    \centering
    \begin{subfigure}[t]{0.23\textwidth}
        \includegraphics[width=\textwidth]{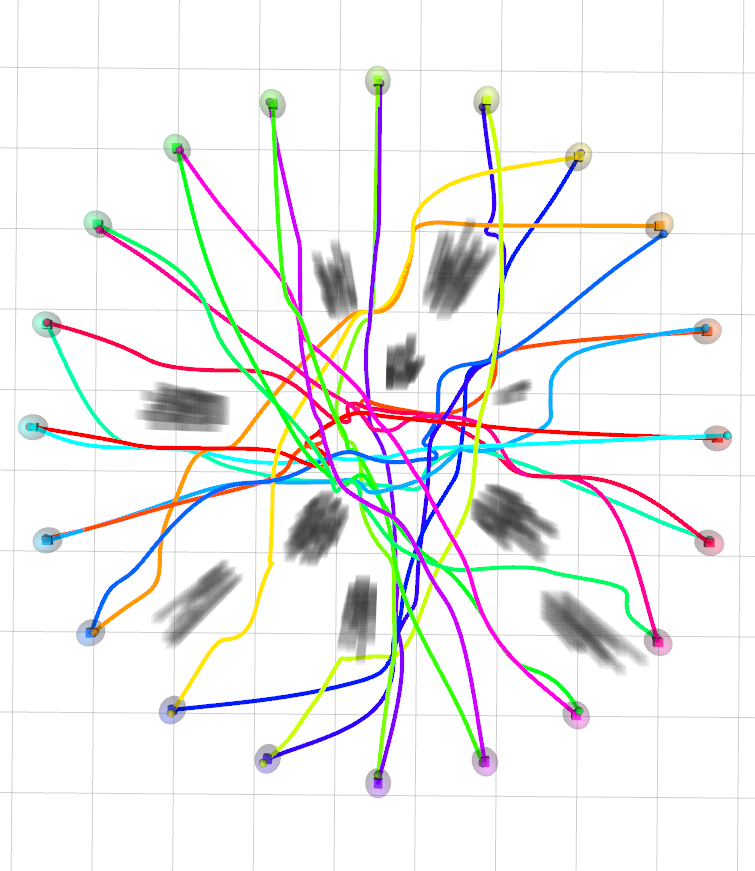}
        \caption{Random forest}
        \label{fig: random forest}
    \end{subfigure}
    ~ 
    \begin{subfigure}[t]{0.20\textwidth}
        \includegraphics[width=\textwidth]{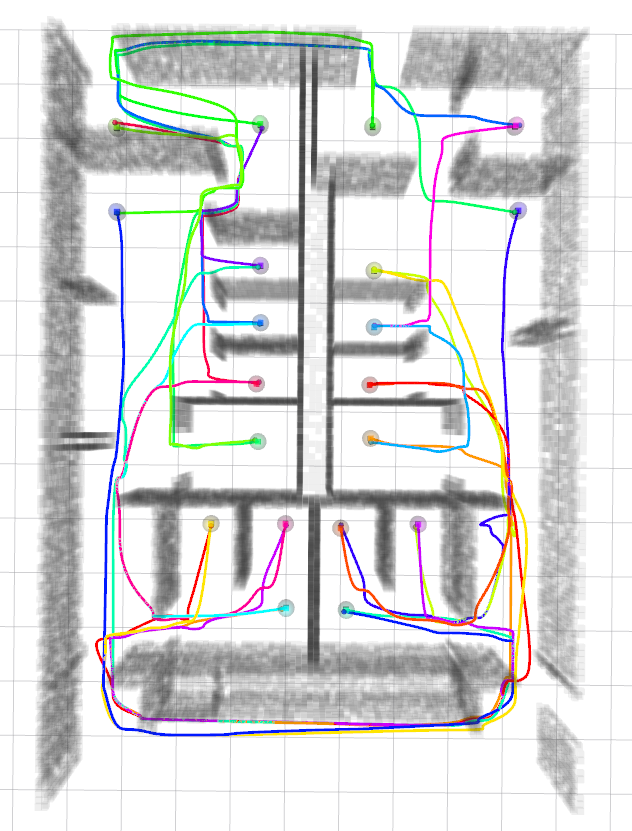}
        \caption{Indoor space}
        \label{fig: indoor space}
    \end{subfigure}
    \caption{Trajectory planning result of the proposed method in obstacle environments. Ellipsoid and line denote the agent and its trajectory respectively, and the gray block is obstacle.}
    \label{fig: trajectory planning result}
\end{figure}

\begin{table}[t!]
\caption{Comparison with EGO-Swarm \cite{zhou2020ego} in cluttered environment. ($sr$: success rate ($\%$), $T_{f}$: flight time (s), $l$: flight distance per agent (m), $T_{c}$: computation time (ms))} \vspace{-3mm}
\label{table: cluttered environment}
\begin{center}
\begin{tabularx}{\linewidth}{c|c|*{4}{X}}
\toprule
Env. & Method & \centering $sr$ & \centering $T_{f}$ & \centering $l$ & \centering $T_{c}$ \tabularnewline
\hline 
\multirow{2}{*}{Forest (20 agents)} &  EGO-Swarm & \centering 50 &  \centering 18.2  & \centering 8.69 & \centering 4.7  \tabularnewline
& LSC-PB & \centering 100 & \centering 18.4 & \centering 8.93 & \centering 9.1  \tabularnewline
\hline  
\multirow{2}{*}{Indoor (20 agents)}  & EGO-Swarm & \centering 0  & \centering -  & \centering - & \centering 53.0 \tabularnewline
& LSC-PB & \centering 100  & \centering 42.0  & \centering 14.6  & \centering 9.2  \tabularnewline
\hline
\end{tabularx}
\end{center}
\vspace{-2mm}
\end{table}

\subsection{Simulation in Cluttered Environments}
We compared the proposed algorithm with the EGO-Swarm \cite{zhou2020ego} in two types of obstacle environment. 
One is a random forest consisting of 10 static obstacles. We deployed 20 agents in a circle with 4 m radius and 1 m height, and we assigned the goal points to antipodes of the start points, as shown in Fig. \ref{fig: random forest}.
The other is an indoor space with the dimension 10 m $\times$ 15 m $\times$ 2.5 m. 
For a fair comparison, we increased the sensor range of EGO-Swarm enough to recognize all obstacles.
We randomly assigned the start and goal points among the 20 predetermined points as depicted in Fig. \ref{fig: indoor space}. We executed 30 simulations for each environment changing the obstacle's position (random forest) or the start and goal points (indoor space). 

Table \ref{table: cluttered environment} shows the planning result of both methods in cluttered environments. 
Ego-Swarm shows faster computation speed in the random forest because it optimizes the trajectory without any hard constraints. However, the success rate of this method is $50\%$ success rate, indicating that EGO-Swarm cannot guarantee collision avoidance.
On the other hand, the proposed method shows a perfect success rate and has similar flight time and distance compared to EGO-Swarm.
This result shows that the LSC is as efficient as the soft constraint of EGO-Swarm but still guarantees collision avoidance in obstacle environments.

Unlike the random forest, EGO-Swarm fails to find the trajectory to the goal points in the indoor space. It is because EGO-Swarm cannot find the topological trajectory detouring the wall-shaped obstacle.
On the other hand, the proposed method reaches the goal for all cases due to the priority-based goal planning.

\subsection{Flight Test}
We conducted the flight test with ten Crazyflie 2.0 quadrotors in the maze-like environment depicted in Fig. \ref{fig: flight test}. 
We confined the quadrotors in 10 m x 6 m x 1 m space to prevent them from bypassing the obstacle region completely.
We assign the quadrotors a mission to patrol the two waypoints to validate the operability of the proposed algorithm.
We sequentially computed the trajectories for all agents in a single laptop, and we used the Crazyswarm \cite{preiss2017crazyswarm} to broadcast the control command to quadrotors.
We utilized the Optitrack motion capture system to observe the agent's position at 120 Hz.
The full flight is presented in the supplemental video.

The flight test was conducted for around 70 seconds, and there was no huge disturbance triggering the planner change. 
We accumulated the history of the smallest inter-agent distance and agent-to-obstacle distance in Fig. \ref{fig: hist_agent} and \ref{fig: hist_obs} respectively.
Fig. \ref{fig: hist_agent} shows that the agents invade desired safe distance for some sample times due to the tracking error. However, there was no physical collision during the entire flight.
The computation time per agent is reported in Fig. \ref{fig: loop_time}, and the average computation time was 6.8 ms.

\begin{figure*}
    \centering
    \begin{subfigure}[t]{0.30\textwidth}
        \centering
        \includegraphics[width=\textwidth]{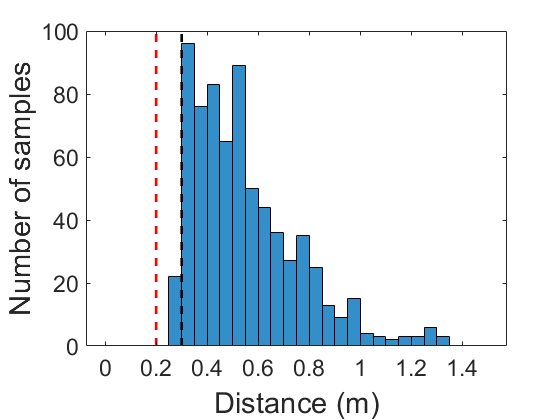}
        \caption{Distance between quadrotors}
        \label{fig: hist_agent}
    \end{subfigure}
    ~ 
    \begin{subfigure}[t]{0.30\textwidth}
        \centering
        \includegraphics[width=\textwidth]{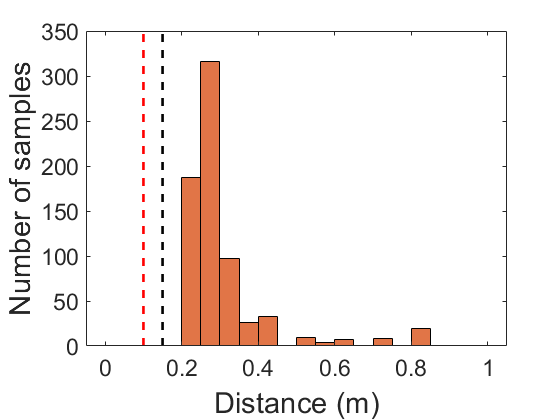}
        \caption{Distance to static obstacles}
        \label{fig: hist_obs}
    \end{subfigure}%
    ~ 
    \begin{subfigure}[t]{0.30\textwidth}
        \centering
        \includegraphics[width=\textwidth]{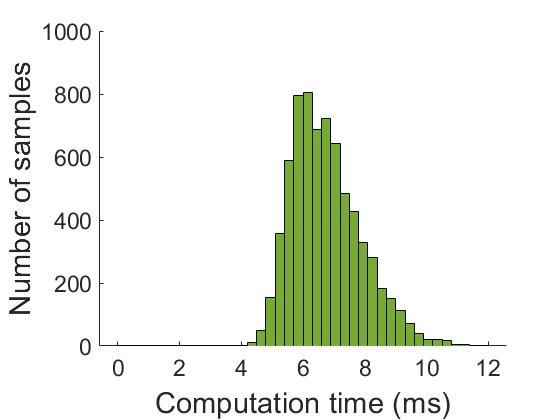}
        \caption{Computation time per agent}
        \label{fig: loop_time}
    \end{subfigure}
    \caption{
     Summary of the flight test with 10 quadrotors in maze-like environment. The red and black dashed lines denote the physical and desired safe distance respectively.
    }
    \label{fig: experiment result}
\end{figure*}

\newpage
\section{CONCLUSIONS}
\label{sec: conclusions}
We presented the online distributed trajectory planning algorithm for quadrotor swarm that guarantees to return collision-free, and dynamically feasible trajectory in cluttered environments.
We constructed the collision avoidance constraints using the linear safe corridor so that the feasible set of the constraints always contains the initial trajectory. 
We utilized these constraints to reformulate the trajectory generation problem as convex optimization, and we proved that this problem ensures feasibility for all replanning steps.
We verified that the prioritized-based goal planning solves deadlock better than the right-hand rule and EGO-Swarm. 
The proposed method shows the highest success rate compared to the state-of-the-art baselines and yields $50\%$ less flight time and $17\%$ less flight distance compared to the BVC-based method. 
We validated the operability of our algorithm by intensive flight tests in a maze-like environment.


\bibliographystyle{./bibtex/IEEEtran}
\bibliography{./bibtex/IEEEabrv, ./bibtex/mybibfile} 

\begin{thebibliography}{10}
\providecommand{\url}[1]{#1}
\csname url@rmstyle\endcsname
\providecommand{\newblock}{\relax}
\providecommand{\bibinfo}[2]{#2}
\providecommand\BIBentrySTDinterwordspacing{\spaceskip=0pt\relax}
\providecommand\BIBentryALTinterwordstretchfactor{4}
\providecommand\BIBentryALTinterwordspacing{\spaceskip=\fontdimen2\font plus
\BIBentryALTinterwordstretchfactor\fontdimen3\font minus
  \fontdimen4\font\relax}
\providecommand\BIBforeignlanguage[2]{{%
\expandafter\ifx\csname l@#1\endcsname\relax
\typeout{** WARNING: IEEEtran.bst: No hyphenation pattern has been}%
\typeout{** loaded for the language `#1'. Using the pattern for}%
\typeout{** the default language instead.}%
\else
\language=\csname l@#1\endcsname
\fi
#2}}

\bibitem{luis2020online}
C.~E. {Luis}, M.~{Vukosavljev}, and A.~P. {Schoellig}, ``Online trajectory
  generation with distributed model predictive control for multi-robot motion
  planning,'' \emph{IEEE Robotics and Automation Letters}, vol.~5, no.~2, pp.
  604--611, 2020.

\bibitem{kandhasamy2020scalable}
S.~Kandhasamy, V.~B. Kuppusamy, and S.~Krishnan, ``Scalable decentralized
  multi-robot trajectory optimization in continuous-time,'' \emph{IEEE Access},
  vol.~8, pp. 173\,308--173\,322, 2020.

\bibitem{zhou2020ego}
X.~Zhou, X.~Wen, J.~Zhu, H.~Zhou, C.~Xu, and F.~Gao, ``Ego-swarm: A fully
  autonomous and decentralized quadrotor swarm system in cluttered
  environments,'' \emph{arXiv preprint arXiv:2011.04183}, 2020.

\bibitem{zhou2021decentralized}
X.~Zhou, Z.~Wang, X.~Wen, J.~Zhu, C.~Xu, and F.~Gao, ``Decentralized
  spatial-temporal trajectory planning for multicopter swarms,'' \emph{arXiv
  preprint arXiv:2106.12481}, 2021.

\bibitem{sharon2015conflict}
G.~Sharon, R.~Stern, A.~Felner, and N.~R. Sturtevant, ``Conflict-based search
  for optimal multi-agent pathfinding,'' \emph{Artificial Intelligence}, vol.
  219, pp. 40--66, 2015.

\bibitem{zhou2017fast}
D.~Zhou, Z.~Wang, S.~Bandyopadhyay, and M.~Schwager, ``Fast, on-line collision
  avoidance for dynamic vehicles using buffered voronoi cells,'' \emph{IEEE
  Robotics and Automation Letters}, vol.~2, no.~2, pp. 1047--1054, 2017.

\bibitem{jungwon2020efficient}
J.~{Park}, J.~{Kim}, I.~{Jang}, and H.~J. {Kim}, ``Efficient multi-agent
  trajectory planning with feasibility guarantee using relative bernstein
  polynomial,'' in \emph{2020 IEEE International Conference on Robotics and
  Automation (ICRA)}, 2020, pp. 434--440.

\bibitem{van2011reciprocal}
J.~Van Den~Berg, S.~J. Guy, M.~Lin, and D.~Manocha, ``Reciprocal n-body
  collision avoidance,'' in \emph{Robotics research}.\hskip 1em plus 0.5em
  minus 0.4em\relax Springer, 2011, pp. 3--19.

\bibitem{arul2020dcad}
S.~H. Arul and D.~Manocha, ``Dcad: Decentralized collision avoidance with
  dynamics constraints for agile quadrotor swarms,'' \emph{IEEE Robotics and
  Automation Letters}, vol.~5, no.~2, pp. 1191--1198, 2020.

\bibitem{wang2017safety}
L.~Wang, A.~D. Ames, and M.~Egerstedt, ``Safety barrier certificates for
  collisions-free multirobot systems,'' \emph{IEEE Transactions on Robotics},
  vol.~33, no.~3, pp. 661--674, 2017.

\bibitem{tordesillas2021mader}
J.~Tordesillas and J.~P. How, ``Mader: Trajectory planner in multiagent and
  dynamic environments,'' \emph{IEEE Transactions on Robotics}, 2021.

\bibitem{park2020online}
J.~Park and H.~J. Kim, ``Online trajectory planning for multiple quadrotors in
  dynamic environments using relative safe flight corridor,'' \emph{IEEE
  Robotics and Automation Letters}, vol.~6, no.~2, pp. 659--666, 2020.

\bibitem{zhu2019b}
H.~Zhu and J.~Alonso-Mora, ``B-uavc: Buffered uncertainty-aware voronoi cells
  for probabilistic multi-robot collision avoidance,'' in \emph{2019
  International Symposium on Multi-Robot and Multi-Agent Systems (MRS)}.\hskip
  1em plus 0.5em minus 0.4em\relax IEEE, 2019, pp. 162--168.

\bibitem{jager2001decentralized}
M.~Jager and B.~Nebel, ``Decentralized collision avoidance, deadlock detection,
  and deadlock resolution for multiple mobile robots,'' in \emph{IEEE/RSJ
  International Conference on Intelligent Robots and Systems.}, vol.~3.\hskip
  1em plus 0.5em minus 0.4em\relax IEEE, 2001, pp. 1213--1219.

\bibitem{desaraju2012decentralized}
V.~R. Desaraju and J.~P. How, ``Decentralized path planning for multi-agent
  teams with complex constraints,'' \emph{Autonomous Robots}, vol.~32, no.~4,
  pp. 385--403, 2012.

\bibitem{csenbacslar2019robust}
B.~{\c{S}}enba{\c{s}}lar, W.~H{\"o}nig, and N.~Ayanian, ``Robust trajectory
  execution for multi-robot teams using distributed real-time replanning,'' in
  \emph{Distributed Autonomous Robotic Systems}.\hskip 1em plus 0.5em minus
  0.4em\relax Springer, 2019, pp. 167--181.

\bibitem{mellinger2011minimum}
D.~Mellinger and V.~Kumar, ``Minimum snap trajectory generation and control for
  quadrotors,'' in \emph{Robotics and Automation (ICRA), 2011 IEEE
  International Conference on}.\hskip 1em plus 0.5em minus 0.4em\relax IEEE,
  2011, pp. 2520--2525.

\bibitem{tordesillas2020minvo}
J.~Tordesillas and J.~P. How, ``Minvo basis: Finding simplexes with minimum
  volume enclosing polynomial curves,'' \emph{arXiv preprint arXiv:2010.10726},
  2020.

\bibitem{zettler1998robustness}
M.~Zettler and J.~Garloff, ``Robustness analysis of polynomials with polynomial
  parameter dependency using bernstein expansion,'' \emph{IEEE Transactions on
  Automatic Control}, vol.~43, no.~3, pp. 425--431, 1998.

\bibitem{tang2016safe}
S.~Tang and V.~Kumar, ``Safe and complete trajectory generation for robot teams
  with higher-order dynamics,'' in \emph{2016 IEEE/RSJ International Conference
  on Intelligent Robots and Systems (IROS)}.\hskip 1em plus 0.5em minus
  0.4em\relax IEEE, 2016, pp. 1894--1901.

\bibitem{boyd2004convex}
S.~Boyd, S.~P. Boyd, and L.~Vandenberghe, \emph{Convex optimization}.\hskip 1em
  plus 0.5em minus 0.4em\relax Cambridge university press, 2004.

\bibitem{gilbert1988fast}
E.~G. Gilbert, D.~W. Johnson, and S.~S. Keerthi, ``A fast procedure for
  computing the distance between complex objects in three-dimensional space,''
  \emph{IEEE Journal on Robotics and Automation}, vol.~4, no.~2, pp. 193--203,
  1988.

\bibitem{hornung2013octomap}
A.~Hornung, K.~M. Wurm, M.~Bennewitz, C.~Stachniss, and W.~Burgard, ``Octomap:
  An efficient probabilistic 3d mapping framework based on octrees,''
  \emph{Autonomous robots}, vol.~34, no.~3, pp. 189--206, 2013.

\bibitem{montanari2018opengjk}
M.~Montanari and N.~Petrinic, ``Opengjk for c, c\# and matlab: Reliable
  solutions to distance queries between convex bodies in three-dimensional
  space,'' \emph{SoftwareX}, vol.~7, pp. 352--355, 2018.

\bibitem{cplex201612}
I.~CPLEX, ``12.7. 0 user’s manual,'' 2016.

\bibitem{preiss2017crazyswarm}
J.~A. Preiss, W.~Honig, G.~S. Sukhatme, and N.~Ayanian, ``Crazyswarm: A large
  nano-quadcopter swarm,'' in \emph{International Conference on Robotics and
  Automation (ICRA)}.\hskip 1em plus 0.5em minus 0.4em\relax IEEE, 2017, pp.
  3299--3304.

\end{thebibliography}

\addtolength{\textheight}{-12cm}   

\end{document}